\documentclass[a4paper,11pt]{article}

\usepackage[toc,page]{appendix}
\usepackage{proof}
\usepackage{graphicx}
\usepackage{amsthm}
\usepackage{amsmath}
\usepackage{amssymb}
\usepackage{mathtools}
\usepackage{latexsym}
\usepackage{amssymb}
\usepackage{verbatim}
\usepackage{setspace}
\usepackage{qtree}
\usepackage{enumitem}
\usepackage{xcolor}
\usepackage{url}
\usepackage{stmaryrd}
\usepackage{listings}
\usepackage[outdir=./]{epstopdf}
\usepackage{hyphenat}
\usepackage{authblk}
\usepackage[margin=1.4in]{geometry}
\usepackage{subfig}
\usepackage{tikz}

\usetikzlibrary{arrows,decorations.pathmorphing,backgrounds,positioning,fit}
\newcommand*{\PL}{\mathrm{PL}}

\newcommand*{\fsize}{\mathit{size}}
\newcommand*{\DM}{\mathrm{DM}}
\newcommand*{\modelset}{\mathcal{M}}
\newcommand*{\formset}{\mathcal{F}}
\newcommand*{\Nset}{\mathbb{N}}
\newcommand*{\MM}{\mathfrak{M}}
\newcommand*{\eofex}{\mbox{}\nobreak\hfill\hspace{0.5em}$\blacksquare$}

\newtheorem{theorem}{Theorem}

\newtheorem{lemma}[theorem]{Lemma}
\newtheorem{corollary}[theorem]{Corollary}
\theoremstyle{definition}
\newtheorem{definition}[theorem]{Definition}
\newtheorem{example}[theorem]{Example}

\begin{document}

\title{Explainability via Short Formulas: the Case of Propositional Logic with Implementation}

\author{
Reijo Jaakkola, Tomi Janhunen, Antti Kuusisto,
Masood~Feyzbakhsh~Rankooh, Miikka Vilander\footnote{The 
authors are given in the alphabetical order. Changes to version one: typos fixed in Section \ref{moregenerally}.}\\

Tampere University, Finland} 

\date{}

\maketitle

\hyphenation{imp-le-men-ta-tion imp-le-men-ta-tions prob-lems}

\begin{abstract}
\noindent
We conceptualize explainability in terms of logic and formula size, giving a number of related definitions of explainability in a very general setting. Our main interest is the so-called special explanation problem which aims to explain the truth value of an input formula in an input model.
The explanation is a formula of minimal size that (1)
agrees with the input formula on the input model and (2)  transmits the involved truth value to the input formula globally, i.e., on every model.
As an important example case, we study propositional logic in this setting and show that the
special explainability problem is complete for the second level of the polynomial hierarchy. We also provide an implementation of this problem in answer set programming and investigate its capacity in relation to explaining answers to the n-queens and dominating set problems.
\end{abstract}


\section{Introduction}

This paper investigates explainability in a general setting. The key in our approach is to relate explainability to formula size. 
We differentiate between the \emph{general} and \emph{special explanation problems}. The general explanation problem for a logic L takes as input a formula $\varphi\in\mathrm{L}$ and outputs an equivalent formula of minimal size. Thus the objective is to explain the global behaviour of $\varphi$. For example, we can consider $\neg\neg \neg \neg\neg \neg p$ to be globally explained by $p$. In contrast, the goal of the special explanation problem is to explicate why an input formula $\psi$ gets the truth value $b$ in an input model $M$. Given a tuple $(M,\psi,b)$, the problem outputs a formula $\chi$ of minimal size such that (1) the formula $\chi$ obtains the same truth value $b$ on $M$, and (2) on every model $M'$ where $\chi$ gets the truth value $b$, also $\psi$ gets that same truth value. Intuitively, this second condition states that the given truth value $b$ of $\chi$ on a model $M'$ causes $\psi$ to be judged similarly. In summary, the special explanation problem gives reasons why a piece of data (or a model) is treated in a given way (i.e., obtains a given truth value) by a classifier (or a formula).

The two explanation problems give rise to corresponding questions of explainability. The \emph{general explainability problem} for a logic L asks, given a formula $\varphi\in\mathrm{L}$ and   $k\in \Nset$, whether there exists a formula equivalent to $\psi$ of size at most $k$. The \emph{special explainability problem} gets as input a tuple $(M,\psi,b,k)$, and the tasks is then to check whether there exists a formula $\chi$ of size at most $k$ satisfying the above conditions (1) and (2) of the special explanation problem.

As an important particular case, we study the special explainability problem of propositional logic (PL) in detail. We prove that the problem is $\Sigma_2^P$-complete. This is an important result whose usefulness lies in its implications that go far beyond PL itself. Indeed, the result gives a robust lower bound for various logics. We demonstrate this by establishing $\Sigma_2^P$-completeness of the special explainability problem of S5 modal logic, and we observe that, indeed, a rather wide range of logics with satisfiability in NP have $\Sigma_2^P$-complete special explainability problems.

As a further theoretical result, we prove that, when limiting to explaining only the positive truth value $\top$, the special explainability problem is only NP-complete for $\mathit{CNF}$-formulas of PL. As a corollary, we get NP-completeness of the problem for $\mathit{DNF}$-formulas in restriction to the truth value $\bot$. While theoretically interesting, these results are also relevant from the point of view of applications, as quite often real-life classification scenarios require explanations only in the case of one truth value. For example, explanations concerning automated insurance decisions are typically relevant only in the case of rejected applications. In addition to the NP-completeness results, we also show that restricting to a single truth value here is necessary for obtaining the given complexity (supposing $\mathrm{coNP}\not\subseteq\mathrm{NP}$).

We provide an implementation of the special explainability problem of PL based on \emph{answer-set programming} (ASP). Generally, ASP is a logic programming language based on stable-model semantics and propositional syntax, particularly Horn clauses. ASP is \emph{especially suitable} and almost custom-made for implementing the special explainability problem of PL, as ASP is designed precisely for the complexity levels up to $\Sigma_2^P$. Indeed, while the disjunction-free fragments of ASP \cite{SimonsNS02} cover the first level of the polynomial hierarchy in a natural way, proper disjunctive rules with cyclic dependencies become necessary to reach the second one \cite{EiterG95}.

We test the implementation via experiments with benchmarks based on the \emph{$n$-queens} and 
\emph{dominating set problems}. The experiments provide concrete and compact explanations why a particular configuration of queens on the generalized chessboard or a particular set of vertices of a graph is or is not an acceptable solution to the involved problem. Runtimes scale exponentially in the size of the instance and negative explanations are harder to compute than positive ones.

Concerning related work, while the literature on explainability is extensive, only a small part of it is
primarily based on logic. In relation to the current paper, the special explainability problem in the particular case of PL has some similarities with the \emph{prime implicant} (or PI) explanations of \cite{ShihCD18}. However, there are some key differences. PI-explanations are defined in terms of finding a minimal subset of features---or propositions---that suffice to explain the input instance of a Boolean decision function. Our definition of special explainability allows for any kind of formula as output. In the propositional case, we separately prove that a subconjunction of the literals in the input is always a possible output. We end up with a similar goal of removing propositions, but from a different starting point and in a different setting. In the case of other logics, such as FO, the resemblance to prime implicant explanations decreases. Concerning results, \cite{ShihCD18} does not provide a full complexity analysis relating to the studied problems. 
Although PI-explanations are defined for any decision function, 
the algorithms used in \cite{ShihCD18} to compute PI-explanations have ordered binary decision diagrams (OBDD) as inputs and outputs. For these algorithms, the authors give empirical runtimes but no complexity bounds. Thus, in our work, while the space of potential input and output formulas is different, we also give a complete
complexity analysis of the special
explainability problem in addition to experiments.

In \cite{Umans01}, Umans shows that the \emph{shortest implicant problem} is $\Sigma_2^p$-complete, thereby solving a long-standing open problem of Stockmeyer. An implicant of $\varphi$ is a conjunction $\chi$ of literals such that $\chi \vDash \varphi$. The shortest implicant problem asks whether there is an implicant of $\varphi$ with size at most $k$. 
Size is defined as the number of occurrences of literals. Below we prove that the special explainability problem for PL can be reduced to certain particular implicant problems. However, despite this, the work of Umans does not directly modify to give the $\Sigma_2^p$-completess result of the special explainability problem. The key issue is that the explainability problem requires an interpolant between a set $\chi$ of literals and a formula $\varphi$, where $\chi$ has precisely the same set of propositions symbols as
$\varphi$. Thus we need to give an independent proof for the $\Sigma_2^P$ lower bound. Also, formula size in \cite{Umans01} is measured in a more coarse way.

\section{Preliminaries}

Let $\Phi$ be a set of proposition symbols. The set $\mathrm{PL}(\Phi)$ of formulas of \textbf{propositional logic} $\mathrm{PL}$ over $\Phi$ is given by the grammar $
\varphi := p \mid \neg\varphi \mid (\varphi \land \varphi) \mid (\varphi \lor \varphi) 
$,
where $p \in \Phi$. A \emph{literal} is either an atom $p$ or its negation $\neg p$, also known as \emph{positive} and \emph{negative} literals, respectively.
%
%
%
%
A $\Phi$-\textbf{assignment} is a
function $s:\Phi\rightarrow \{0,1\}$. When $\Phi$ is clear from the context or irrelevant, we simply refer to assignments rather than $\Phi$-assignments. We define the
semantics of
propositional logic in the usual way, and we
write $s\models \varphi$ if the assignment $s$
\textbf{satisfies} the formula $\varphi\in \mathrm{PL}(\Phi)$. Alternatively, we 
can use the \textbf{standard valuation}
function $v_{\Phi}$ defined
such that $v_{\Phi}(s,\varphi) = 1$ if
$s\vDash \varphi$ and
otherwise $v_{\Phi}(s,\varphi) = 0$.
Below we sometimes use the set $\{\bot,\top\}$ instead of $\{0,1\}$,


A formula $\psi \in \PL(\Phi)$ is a \textbf{logical consequence} of $\varphi \in \PL(\Phi)$, denoted $\varphi \vDash \psi$, if for every $\Phi$-assignment $s$, $s \vDash \varphi$ implies $s \vDash \psi$. A formula $\chi \in \PL(\Phi)$ is an \textbf{interpolant} between $\varphi$ and $\psi$ if $\varphi \vDash \chi$ and $\chi \vDash \psi$.
For a finite $\Phi$, we say that a formula $\varphi$ is a \textbf{maximal conjunction} w.r.t. $\Phi$ if $\varphi$ is a conjunction of exactly one \emph{literal} for each $p \in \Phi$. \emph{A $\Phi$-assignment $s$
can be naturally identified with 
a maximal conjunction,} for example 
$\{(p,1),(q,0)\}$ identifies with
$p\wedge\neg q$.
A formula $\chi$ is a \textbf{subconjunction} of $\varphi$ if $\chi$ is a conjunction of literals occurring in $\varphi$. 
The \textbf{size} of $\varphi$, denoted $\fsize(\varphi)$, is the number of occurrences of proposition symbols, binary connectives and negations in $\varphi$. For example, the size of $\neg \neg (p\wedge p)$ is $5$ as it has one $\wedge$ and two
occurrences of 
both $\neg$ and $p$. 

\section{Notions of explanation and explainability}

In this section we introduce four
natural problems concerning the
general and special perspectives to explainability. The general problems deal with the question of explaining the entire behaviour of a classifier, whereas the special ones attempt to explicate why a single input
instance was classified in a given way.
We give very general definitions of these problems, and for that we will devise a very general definition of the notion of a logic. Our definition of a logic covers various kinds of classifiers in addition to standard formal logics, including logic programs, Turing machines, neural network models, automata, and the like.

\begin{definition}
We define that a \textbf{logic} is a tuple $(\modelset, \formset, v, m)$ where $\modelset$ and $\formset$ are sets; 
${v:\mathcal{M}\times \mathcal{F}\rightarrow V}$ is a
function mapping to some set $V$; and 
$m : \formset \to \Nset$ is a function. Emphasizing 
the set $V$, we can 
also call $(\modelset, \formset, v, m)$ a \textbf{$V$-valued logic}.
\end{definition}

Intuitively, we can think of $\mathcal{M}$ as a 
set of models and $\mathcal{F}$ as a set of formulas.
The function $v:\mathcal{M}\times\mathcal{F}
\rightarrow V$ gives the
semantics of the logic, with
$v(\mathfrak{M},\varphi)$ being the truth value of $\varphi$ in $\mathfrak{M}$. We
call $v$ a \textbf{valuation}. The function $m$ 
gives a complexity measure for the formulas in $\mathcal{F}$,
such as, for example, formula size. In computational problems, where members of $V$ are inputs, we of course assume that $V$ is somehow representable.

\begin{example}
Propositional logic PL over a 
set $\Phi$ of proposition symbols 
can be defined as a tuple $(\mathcal{M},
\mathcal{F},v,m)$, where 
$\mathcal{M}$ is the set of $\Phi$-assignments; 
$\mathcal{F}$ the set $\mathrm{PL}(\Phi)$ of
formulas; $v:\mathcal{M}\times\mathcal{F}
\rightarrow \{0,1\}$ is
the standard valuation $v_{\Phi}$; and
$m(\varphi) = \fsize(\varphi)$
for all $\varphi\in \mathcal{F}$.
\eofex
\end{example}

Now, the following example demonstrates that we can consider much more general
scenarios than ones involving
the standard formal logics.

\begin{example}
Let $\mathcal{M}$ be a set of data 
and $\mathcal{F}$ a set of programs
for classifying the data, that is, 
programs that take elements of $\mathcal{M}$ as
inputs and output a value in
some set $V$ of suitable outputs.
Now $v:\mathcal{M}\times\mathcal{F}
\rightarrow V$ is just the function 
such that $v(D,P)$ is the output of $P\in\mathcal{F}$ on
the input $D\in\mathcal{M}$.
The function $m$ can quite
naturally give the program size for
each $P\in\mathcal{F}$. If we redefine the domain of $m$ to be $\mathcal{M}\times\mathcal{F}$, we can let
$m(D,P)$ be for example the running time of the program $P$ on the input $D$, or the length of the 
computation (or derivation) table. 
\eofex
\end{example}

Given a logic, we define the 
equivalence relation $\equiv\ \subseteq\mathcal{F}\times
\mathcal{F}$ such that $(M,M')\in\ \equiv$ if 
and only if 
$v(M,\varphi) = v(M',\varphi)$
for all $\varphi\in\mathcal{F}$.
We shall now define four formal
problems relating to explainability. 
The problems do work 
especially well for finite $V$, but this is
not required as long as the elements of $V$ 
are representable in the sense that they can be
used as inputs to computational problems.

\begin{definition}\label{firstdefexpl} Let $L = (\modelset, \formset, v, m)$ be a logic.
We define the following four problems for $L$.

\noindent\textbf{General explanation problem:}\\
\emph{Input}: $\varphi \in \formset$,\,  \emph{Output}: $\psi \in \formset$\\
\emph{Description}: Find $\psi \in \formset$ with $\psi \equiv \varphi$ and minimal $m(\psi)$.

\smallskip

\noindent\textbf{Special explanation problem} \\
\emph{Input}: $(\mathfrak{M},\varphi,b)$ 
where $\MM \in \modelset$, $\varphi \in \formset$ and $b \in V$,\, 
\emph{Output}: $\psi \in \formset$ or $\mathtt{error}$ \\ 
\emph{Description}: If $v(\mathfrak{M},\varphi)\not= b$,
output $\mathtt{error}$. Else
find $\psi\in\mathcal{F}$ with minimal $m(\psi)$
such that the following two conditions hold: \\
\indent {(1)}\  $v(\mathfrak{M},\psi) = b$ and\\ 
\indent {(2)} For all $\mathfrak{M}'
    \in \mathcal{M}$, $v(\mathfrak{M}',\psi) = b\  \Rightarrow\ v(\mathfrak{M}',\varphi) = b$.

\smallskip

\noindent\textbf{General explainability problem} \\
\emph{Input}: $(\varphi, k)$, where $\varphi \in \formset$ and $k \in \Nset$, \,
\emph{Output}: Yes or no \\
\emph{Description}: If there is $\psi \in \formset$ with $\psi \equiv \varphi$ and $m(\psi) \leq k$, output yes. Otherwise output no.

\smallskip

\noindent\textbf{Special explainability problem} \\
\emph{Input}: $(\mathfrak{M},\varphi,b, k)$ 
where $\MM \in \modelset$, $\varphi \in \formset$, $b \in V$ and $k \in \Nset$, \, 
\emph{Output}: Yes or no \\ 
\emph{Description}: Output ``yes'' if and only if
there exists some $\psi\in\mathcal{F}$
with $m(\psi)\leq k$ such 
that the conditions (1) and (2) of the 
special explanation problem hold.

\end{definition}


The above definitions are quite flexible. Notice, for example, that while the set $\mathcal{M}$ may typically be considered a set of models, or pieces of data, there are many further natural possibilities. For
instance, $\mathcal{M}$ can be a set of formulas. This nicely covers, e.g., model-free
settings based on proof systems.
However, the definitions above generalize quite naturally to a yet more 
comprehensive setting, as we will next demonstrate. The
main reason that we have first given the above definitions is that in this article we
shall concentrate
mostly on the case of propositional logic, and the definitions in their
above form suffice for that purpose. Secondly, the more general definitions
below are easier to
digest when reflected and compared to the above simpler scheme.
In the following subsection we 
develop the more general definitions. As already noted, we shall stick to the
above, less general definitions in the remaining parts of the article,
starting from Section \ref{explainabilityforpl}.

\subsection{The explanation and explainability problems more generally}\label{moregenerally}

The definition of a logic can be naturally extended as follows, allowing the complexity measure $m$ to depend on models as well as formulas.

\begin{definition}\label{logictwo}
A \textbf{logic} is a tuple of the form $(\modelset, \formset, v, m)$ where $\modelset$ and $\formset$ are sets; 
${v:\mathcal{M}\times \mathcal{F}\rightarrow V}$ is a
function mapping to some set $V$; and 
$m : \modelset \times \formset \to \Nset$ is a function. Emphasizing 
the set $V$, we can 
also call $(\modelset, \formset, v, m)$ a \textbf{$V$-valued logic}.
\end{definition}

We first generalize special explainability by replacing the single truth value $b \in V$ with a set $B \subseteq V$. Supposing $\mathcal{V}$ to be a
finite or countably infinite set of symbols, we call a
function $w:\mathcal{V}\rightarrow \mathcal{P}(V)$ a \textbf{representation over $\mathcal{P}(V)$}.
Here $\mathcal{P}(V)$ denotes the power set of $V$.  
For $b\in V$ and $B\in \mathcal{V}$, we write $b \in B$ to mean that $b \in w(B)$. The special explanation and special explainability problems are then defined as follows:

\begin{definition}\label{setdefexpl}
Let $L = (\modelset, \formset, v, m)$ be a logic.
We define the following two problems for $L$. 

\noindent\textbf{Special explanation problem} \\
\emph{Input}: $(\mathfrak{M},\varphi,B)$ 
where $\MM \in \modelset$, $\varphi \in \formset$ and $B \in \mathcal{V}$,\, 
\emph{Output}: $\psi \in \formset$ or $\mathtt{error}$ \\ 
\emph{Description}: If $v(\mathfrak{M},\varphi)\notin B$,
output $\mathtt{error}$. Else
find $\psi\in\mathcal{F}$ with minimal $m(\MM,\psi)$
such that the following two conditions hold: \\
\indent {(1)}\  $v(\mathfrak{M},\psi) \in B$ and\\ 
\indent {(2)} For all $\mathfrak{M}'
    \in \mathcal{M}$, $v(\mathfrak{M}',\psi) \in B\  \Rightarrow\ v(\mathfrak{M}',\varphi) \in B$.

\smallskip

\noindent\textbf{Special explainability problem} \\
\emph{Input}: $(\mathfrak{M},\varphi,B, k)$ 
where $\MM \in \modelset$, $\varphi \in \formset$, $B \in \mathcal{V}$ and $k \in \Nset$, \, 
\emph{Output}: Yes or no \\ 
\emph{Description}: Output ``yes'' if and only if
there exists some $\psi\in\mathcal{F}$
with $m(\MM,\psi)\leq k$ such 
that the conditions (1) and (2) of the 
special explanation problem hold.
\end{definition}

\begin{example}
    Let $\mathcal{M} = V = \mathbb{Z}$ and let $\mathcal{F}$ be the set of polynomials in the variable $x$ with integer coefficients. The function $v: \mathcal{M} \times \mathcal{F} \to V$ evaluates a polynomial $p \in \mathcal{F}$ at a point $a \in \mathcal{M}$. Let $m: \mathcal{F} \to \mathbb{N}$ be the complexity measure which counts the number of occurrences of $x$, operations (addition and multiplication) and constants.\footnote{Note that here we do not utilize the possibility of letting $m$ depend also on $\mathcal{M}$.}
    
    Consider the special explanation problem with the input 
    \[
    (4, p = x^5-2x^4-x^3-5x^2-2x-3, \mathbb{Z}_{> 0}).
    \]
    Intuitively, we are asking for an explanation for the fact that this polynomial evaluated at 4 gets a positive value. In this case the problem would output $x-3$, a very simple explanation. We see that $4-3 = 1$ is positive and since $p = (x-3)(x^2+1)(x^2+x+1)$, we also see that whenever $x-3$ is positive, $p$ is positive as well. Finally $m(x-3) = 3$ so $x-3$ is clearly the minimal explanation.
    \eofex
\end{example}

We generalize the problems further to allow for uncertainty or approximation. We consider metric spaces over 
the sets $V$ of ``truth values'' in $V$-valued logics. To enable real numbers in 
inputs to computational problems, we also consider representable
sets of reals: if $\mathcal{R}$ is a finite or countably infinite set of 
symbols, then a function $r:\mathcal{R}\rightarrow\mathbb{R}$ is called a
\textbf{representation of reals}. If $\varepsilon\in \mathcal{R}$, we
let $\mathit{Ball}(b,\varepsilon)$ denote the open ball that 
contains all those points of $V$ whose distance is at most $r(\varepsilon)$
from the point $b\in V$.\footnote{For the case $r(\varepsilon) = 0$ we consider a singleton to be an open ball. This is to make sure that our definitions generalize the previous ones without uncertainty or approximation.}

For the general explanation problem we additionally consider a pseudometric $d: \mathcal{F} \times \mathcal{F} \to \mathbb{R}_{\geq 0}$ over $\mathcal{F}$. A pseudometric is defined like a metric with the condition $d(x,y) = 0 \Rightarrow x = y$ removed. The pseudometric $d$ corresponds to the degree of equivalence between two formulas. Since two syntactically different formulas can nevertheless be equivalent, the use of a pseudometric instead of just a metric is well-justified.

The below definition extends Definitions \ref{firstdefexpl} and \ref{setdefexpl}.

\begin{definition} Let $L = (\modelset, \formset, v, m)$ be a $V$-valued logic as
given in Definition \ref{logictwo}. Let $\mathcal{T}$ be a 
metric space over $V$ and $r:\mathcal{R}\rightarrow\mathbb{R}$ a 
representation of reals. Let $w:\mathcal{V}\rightarrow \mathcal{P}(V)$ be a
representation over $\mathcal{P}(V)$. Let $d: \mathcal{F} \times \mathcal{F} \to \mathbb{R}_{\geq 0}$ be a pseudometric over $\mathcal{F}$.
We define the following four problems for $L$.

\medskip

\noindent\textbf{General explanation problem:}\\
\emph{Input}: $(\varphi, \varepsilon)$, where $\varphi \in \formset$, $\varepsilon \in \mathcal{R}$ \,  \emph{Output}: $\psi \in \formset$\\
\emph{Description}: Find $\psi \in \formset$ with $d(\varphi, \psi) \leq r(\varepsilon)$ and minimal $m(\MM,\psi)$.

\medskip

\noindent\textbf{Special explanation problem} \\
\emph{Input}: $(\mathfrak{M},\varphi,B,\varepsilon)$ 
where $\MM \in \modelset$, $\varphi \in \formset$, $B \in \mathcal{V}$
and $\varepsilon \in \mathcal{R}$\\ 
\emph{Output}: $\psi \in \formset$ or $\mathtt{error}$ \\ 
\emph{Description}: If $v(\mathfrak{M},\varphi)\not\in B$,
output $\mathtt{error}$.
Else find $\psi\in\mathcal{F}$ with minimal $m(\mathfrak{M}, \psi)$
such that the following two conditions hold: \\ \\ 
\indent {(1)}\  $v(\mathfrak{M},\psi) \in \bigcup\limits_{b\in B}\mathit{Ball}(b,\varepsilon)$ and\\ 
\indent {(2)} For all $\mathfrak{M}'
    \in \mathcal{M}$, we have 
$$v(\mathfrak{M}',\psi) \in \bigcup\limits_{b\in B}\mathit{Ball}(b,\varepsilon)\  \Rightarrow\ v(\mathfrak{M}',\varphi) \in \bigcup\limits_{b\in B}\mathit{Ball}(b,\varepsilon).$$

\medskip

\noindent\textbf{General explainability problem} \\
\emph{Input}: $(\varphi, \varepsilon,k)$, where $\varphi \in \formset$, $\varepsilon \in \mathcal{R}$ and $k \in \Nset$, \,
\emph{Output}: Yes or no \\
\emph{Description}: If there is $\psi \in \formset$ with $d(\varphi, \psi) \leq r(\varepsilon)$ and $m(\MM,\psi) \leq k$, output yes. Otherwise output no.

\medskip

\noindent\textbf{Special explainability problem} \\
\emph{Input}: $(\mathfrak{M},\varphi,B, \varepsilon, k)$ 
where $\MM \in \modelset$, $\varphi \in \formset$, $B \in \mathcal{V}$, $\varepsilon \in \mathcal{R}$ and $k \in \Nset$, \, \\
\emph{Output}: Yes or no \\ 
\emph{Description}: Output ``yes'' if and only if
there exists some $\psi\in\mathcal{F}$
with $m(\mathfrak{M},\psi)\leq k$ such 
that the conditions (1) and (2) of the 
special explanation problem hold.

\end{definition}

Note that if we restrict the sets $B$ in the above definition to singletons, we obtain an important, specific generalization of the special explanation and special explainability problems of Definition \ref{firstdefexpl}. In this version, there is only one truth value to be explained but we still have a metric space over the truth values.

    Another important generalization is the possibility of the input and output logics being different, for example explaining the behaviour of a neural network via a small decision tree. For two $V$-valued logics, the definition of the special explainability problem works as is. For general explainability, the pseudometric $d$ must be defined over the set $\mathcal{F}_1 \cup \mathcal{F}_2$ of formulas of both logics.

\subsection{Special explainability for PL}\label{explainabilityforpl}
%
%
%

Let $(\varphi,\psi,b)$ be an input to the special explanation problem of propositional logic, where $\varphi$ is a maximal conjunction w.r.t. some (any) finite $\Phi$ (thus encoding a $\Phi$-assignment), $\psi$ a $\Phi$-formula and $b\in \{\bot,\top\}$. The special explanation problem can be reformulated equivalently in the following way. 
(1) Suppose $b = \top$. If $\varphi \models \psi$, find a minimal interpolant between $\varphi$ and $\psi$. Else
output $\mathtt{error}$.
(2) Suppose $b = \bot$. If $\varphi \models \neg \psi$, find a minimal interpolant between $\psi$ and $\neg \varphi$.
Else output $\mathtt{error}$.

Let $\varphi \in \PL(\Phi)$ be a conjunction of literals. Let $P(\varphi)$ and $N(\varphi)$ be the sets of positive and negative literals in $\varphi$, respectively. We denote the De Morgan transformations of $\varphi$ and $\neg\varphi$ by
\[
\DM(\varphi) := \bigwedge\limits_{p \in P(\varphi)} p \land \neg \bigg(\bigvee\limits_{\neg q \in N(\varphi)} q\bigg) \text{ \ \ \ \ and \ \ \ \ } \DM(\neg \varphi) := \neg \bigg(\bigwedge_{p \in P(\varphi)} p\bigg) \lor \bigvee_{\neg q \in N(\varphi)} q.
\]


\begin{lemma}\label{fullsize}
Let $\Phi$ be a finite set of proposition symbols, let $\varphi$ be a maximal conjunction w.r.t. $\Phi$ and let $\psi \in \PL(\Phi)$. (1) If $\varphi \models \psi$, then there is a subconjunction $\chi$ of $\varphi$ such that $\DM(\chi)$ is a minimal interpolant between $\varphi$ and $\psi$.
(2) If $\varphi \models \neg \psi$, then there is a subconjunction $\chi$ of $\varphi$ such that $\DM(\neg \chi)$ is a minimal interpolant between $\psi$ and $\neg \varphi$.

\end{lemma}
\begin{proof}
Assume that $\varphi \models \psi$. Clearly at least one minimal interpolant exists, as for example $\varphi$ itself is an interpolant. Let $\theta$ be a minimal interpolant. Let $\Phi(\theta)$ be the set of proposition symbols occurring in $\theta$. We transform $\theta$ into an equivalent formula $\theta'$ in {$\Phi(\theta)$-full} disjunctive normal form where each disjunct is a maximal conjunction w.r.t. $\Phi(\theta)$. 
To see that such a form always exists, first consider the DNF of $\theta$ and then, for each disjunct, fill in all possible values of any missing propositions (thus possibly increasing the number of disjuncts).

Now, as $\varphi$ is a maximal conjunction w.r.t. $\Phi$, exactly one disjunct of $\theta'$ is a subconjunction of $\varphi$. Let $\chi$ denote this disjunct. As $\varphi \models \psi$,
the formula $\chi$ is clearly an interpolant between $\varphi$ and $\psi$. 

Now, each proposition in $\Phi(\theta)$ occurs in $\chi$ and thus also in $\DM(\chi)$ exactly once, so $\DM(\chi)$ has at most the same number of occurrences of proposition symbols and binary connectives as $\theta$. 
Furthermore, $\DM(\chi)$ has at most one negation.
If $\theta$ has no negations, then we claim $\DM(\chi)$ also has none. To see this, note that $\chi$ is a disjunct of $\theta'$ and $\theta'$ is equivalent to $\theta$, so $\chi\models\theta$. As $\theta$ is negation-free, we have $\chi'\models\theta$ where $\chi'$ is obtained from $\chi$ by removing all the negative literals. Hence, by the minimality of $\theta$, we have $\chi' = \chi$. Thus $\chi$ and $\DM(\chi)$ are negation-free. 
We have shown that $\fsize(\DM(\chi)) \leq \fsize(\theta)$, so $\DM(\chi)$ is a minimal interpolant.

Suppose then that $\varphi \vDash \neg\psi$. Let $\theta$ be a minimal interpolant between $\psi$ and $\neg \varphi$. In a dual fashion compared to the positive case above, we transform $\theta$ into an equivalent formula $\theta'$ in {$\Phi(\theta)$-full} conjunctive normal form. Additionally let $\varphi'$ be the negation normal form of $\neg\varphi$. Now $\varphi'$ is a disjunction of literals and exactly one conjunct of $\theta'$ is a subdisjunction of $\varphi'$. Let $\chi'$ denote this conjunct and let $\chi$ denote the negation normal form of $\neg \chi'$. Now $\chi$ is a subconjunction of $\varphi$ and $\neg \chi$ is an interpolant between $\psi$ and $\neg \varphi$. 

As in the positive case, $\DM(\neg\chi)$ has at most the same number of proposition symbols and binary connectives as the minimal interpolant $\theta$. The formula $\DM(\neg\chi)$ again has at most one negation so we only check the case where $\theta$ has no negations. Recall that $\neg\chi$ is equivalent to $\chi'$, which in turn is a conjunct of $\theta'$. Thus $\theta \vDash \chi'$. As $\theta$ has no negations and $\chi'$ is a disjunction of literals, we have $\theta \vDash \chi''$, where $\chi''$ is obtained from $\chi'$ by removing all the negative literals. By the minimality of $\theta$ we obtain $\chi'' = \chi'$ so $\chi'$ has no negations. Thus also $\DM(\neg\chi)$ is negation-free and is a minimal interpolant.
\end{proof}

The above lemma implies that for propositional logic, it suffices to consider subconjunctions of the input in the special explanation and explainability problems. This will be very useful both in the below theoretical considerations and in implementations.

We next prove $\Sigma_2^p$-completeness of the special explainability problem via a reduction from $\Sigma_2 SAT$, which is well-known to be $\Sigma_2^p$-complete. The input of the problem is a quantified Boolean formula $\varphi$ of the form 
$$\exists p_1 \dots \exists p_n \forall q_1 \dots \forall q_m \theta(p_1,\dots,p_n,q_1,\dots,q_m).$$
The output is yes iff $\varphi$ is true.

\begin{theorem}\label{maintheoreticalresult}
The special explainability problem for $\PL$ is $\Sigma_2^p$-complete.
\end{theorem}
\begin{proof}
The upper bound is clear. For the lower bound, we will give a polynomial time (Karp-) reduction from $\Sigma_2 SAT$. Consider an instance
$$\exists p_1 \dots \exists p_n \forall q_1 \dots \forall q_m \theta(p_1,\dots,p_n,q_1,\dots,q_m)$$
of $\Sigma_2 SAT$. We start by introducing, for every existentially quantified Boolean variable $p_i$, a new proposition symbol $\overline{p}_i$. 
Denoting $\theta(p_1,\dots,p_n,q_1,\dots,q_m)$ simply by $\theta$, we define 
\[\psi\ :=\ \bigwedge_{i=1}^n (p_i \lor \overline{p}_i) \land \bigg(\theta \vee \bigvee_{i = 1}^n (p_i \land \overline{p}_i)\bigg).\]
We let $s$ be the valuation mapping all proposition symbols to 1, i.e., the assignment corresponding to the maximal conjunction $\varphi_s := \bigwedge_{i=1}^n (p_i \land \overline{p}_i) \land \bigwedge_{j=1}^m q_j$ w.r.t. the set of proposition symbols in $\psi$.
%
%
%
%
%
%
%
%
%
Clearly $\varphi_s \models \psi$. We now claim that there exists an interpolant of size at most $2n - 1$ between $\varphi_s$ and $\psi$ iff the original instance of $\Sigma_2 SAT$ is true.

Suppose first that the original instance of $\Sigma_2 SAT$ is true. Thus there exists a tuple $(p_1,\dots,p_n) \in \{0,1\}^n$ such that
$\forall q_1 \dots q_m \theta(p_1,\dots,p_n,q_1,\dots,q_m)$
is true. Consider now the subconjunction
$\chi := \bigwedge_{s(p_i) = 1} p_i\ \land\ \bigwedge_{s(p_i) = 0} \overline{p}_i$ of $\varphi_s$.
%
%
%
%
%
%
Clearly $\chi$ is of size $2n - 1$. It is easy to see that $\chi$ is also an interpolant between $\varphi_s$ and $\psi$.

Suppose then that $\chi$ is an interpolant of size at most $2n - 1$ between $\varphi_s$ and $\psi$. Using Lemma \ref{fullsize}, we can assume that $\chi$ is a subconjunction of $\varphi_s$. Since $\chi$ has size at most $2n - 1$, it can contain at most $n$ proposition symbols. Furthermore, $\chi$ must contain, for every $i\in\{1,\dots , n\}$, either $p_i$ or $\overline{p}_i$, since otherwise $\chi$ would not entail $\bigwedge_{i=1}^n (p_i \lor \overline{p}_i)$. Thus $\chi$ contains precisely $n$ proposition symbols. More specifically, $\chi$ contains, for every $i\in \{1,\dots , n\}$, either $p_i$ or $\overline{p}_i$. Now, we define a tuple $(u_1,\dots,u_n) \in \{0,1\}^n$ by setting $u_i = 1$ if $\chi$ contains $p_i$ and $u_i = 0$ if $\chi$ contains $\overline{p}_i$. It is easy to see that
$\forall q_1 \dots q_m \theta(u_1,\dots,u_n,q_1,\dots,q_m)$
is true.
\end{proof}

The above theorem immediately implies a wide range of corollaries. 
Recall that $\mathrm{S5}$ is the system of modal logic where the accessibility relations are equivalences. The special explainability problem for $\mathrm{S5}$ has as input a pointed $\mathrm{S5}$-model $(\mathfrak{M}, w)$, an $\mathrm{S5}$-formula $\varphi$, $b \in \{\top, \bot\}$ and $k \in \Nset$.

\begin{corollary}
The special explainability problem for $\mathrm{S5}$ is $\Sigma_2^p$-complete.
\end{corollary}
\begin{proof}
The lower bound follows immediately from Theorem \ref{maintheoreticalresult}, while the upper bound follows from the well-known fact that the validity problem for $\mathrm{S5}$ is \textsc{coNP}-complete \cite{BlackburnRV01}.
\end{proof}

Note indeed that the $\Sigma_2^p$ lower bound for propositional logic is a rather useful result, implying $\Sigma_2^p$-completeness of special explainability for various logics with an $\mathrm{NP}$-complete satisfiability.

We next show that if the formula $\psi$ in the special explainability problem is restricted to $\mathit{CNF}$-formulas and we consider only the case $b = \top$, then the problem is \textsc{NP}-complete. The following example demonstrates that it is necessary to restrict to the case $b = \top$.

\begin{example}\label{example:cnf-generally-hard}
    Let $\psi \in \PL(\Phi)$ be an arbitrary $\mathit{CNF}$-formula and let $q$ be a proposition symbol such that $q \not \in \Phi$. Suppose that $s \nvDash \psi$, where $s(p) = 1$ for every $p \in \Phi$. Consider now the formula $\varphi_q := q \ \lor \ \bigvee_{p \in \Phi} \neg p$.
    Note that $\neg \varphi_q$ is equivalent to a maximal conjunction w.r.t. $\Phi \cup \{q\}$. Clearly $\psi \land \neg q \models \varphi_q$, since $\psi$ entails that $\bigvee_{p \in \Phi} \neg p$. We now claim that there exists an interpolant $\chi$ of size one between $\psi \land \neg q$ and $\varphi_q$ iff $\psi$ is unsatisfiable. First, we note that the only proposition from $\Phi \cup \{q\}$ which entails $\varphi_q$ is $q$. But the only way $q$ can be an interpolant between $\psi \land \neg q$ and $\varphi_q$ is that $\psi$ is unsatisfiable. Conversely, if $\psi$ is unsatisfiable, then $q$ is clearly an interpolant between $\psi \land \neg q$ and $\varphi_q$. Since the unsatisfiability problem of $\mathit{CNF}$-formulas is \textsc{coNP}-hard, the special explainability problem for $\mathit{CNF}$-formulas is also \textsc{coNP}-hard.
    \eofex
\end{example}

To prove \textsc{NP}-hardness, we will give a reduction from the dominating set problem. For a graph $G = (V, E)$, a \textbf{dominating set} $D \subseteq V$ is a set of vertices such that every vertice not in $D$ is adjacent to a vertice in $D$. The input of the dominating set problem is a graph and a natural number $k$. The output is yes, if the graph has a dominating set of at most $k$ vertices.

\begin{theorem}\label{thm:cnf-formulas}
For $\mathit{CNF}$-formulas, the special explainability problem with $b=\top$ is \textsc{NP}-complete. The lower bound holds even if we restrict our attention to formulas without negations. 
\end{theorem}
\begin{proof}
For the upper bound, let $\psi \in \PL(\Phi)$ be a $\mathit{CNF}$-formula and let $\varphi$ be a maximal conjunction w.r.t. $\Phi$. We want to determine whether there is an interpolant of size at most $k$. Using Lemma \ref{fullsize}, it suffices to determine whether there is a subconjunction $\chi$ of $\varphi$ such that $\fsize(\DM(\chi)) \leq k$.

Our nondeterministic procedure will start by guessing a subconjunction $\chi$ of $\varphi$. If $\fsize(\DM(\chi)) > k$, then it rejects. Otherwise we replace the formula $\psi$ with the formula $\psi'$ which is obtained from $\psi$ by replacing each proposition symbol $p$ that occurs in $\chi$ with either $\top$ or $\bot$, depending on whether $p$ occurs positively or negatively in $\chi$. Now, if $\psi'$ is valid, then our procedure accepts, and if it is not, then it rejects. Since the validity of $\mathit{CNF}$-formulas can be decided in polynomial time, our procedure runs in polynomial time as well.

For the lower bound we will give a reduction from the dominating set problem. Consider a graph $G = (V,E)$ and a parameter $k$. Let 
\begin{equation}\label{eq:dominating}
\psi := \bigwedge_{v \in V} \bigg(p_v \lor \bigvee_{(v,u) \in E} p_u \bigg)
\end{equation}
and $\varphi := \bigwedge_{v \in V} p_v$. Now $\varphi \models \psi$ and $\psi$ is a $\mathit{CNF}$-formula. It is easy to verify that there exists an interpolant $\theta$ of size at most $2k - 1$ if and only if $G$ has a dominating set of size at most $k$.
\end{proof}

By simply negating formulas, we obtain that the special explainability problem with $b = \bot$ is \textsc{NP}-complete for $\mathit{DNF}$-formulas (and in general \textsc{coNP}-hard, see Example \ref{example:cnf-generally-hard}).

\begin{corollary}\label{cor:dnf-formulas}
    For $\mathit{DNF}$-formulas, the special explainability problem with $b = \bot$ is \textsc{NP}-complete.
\end{corollary}

\subsection{On general explainability}

The general explainability problem for 
propositional logic has 
been discussed in the literature under 
motivations unrelated to 
explainability. The
\textbf{minimum equivalent 
expression problem} MEE asks, given a
formula $\varphi$ and an integer $k$,
if there exists a formula equivalent to $\varphi$ and of
size at most $k$.
This problem has been shown in \cite{BuchfuhrerU11} to be $\Sigma^p_2$-complete under 
Turing reductions, with 
formula size defined as the number of occurrences of proposition symbols and with formulas in
negation normal form. 
The case of standard reductions is open.

For logics beyond PL, the literature on the complexity of formula minimization is surprisingly scarce. The study of formula size in first-order and modal logics has mainly focused on particular properties that either can be expressed very succinctly or via a very large formulas. This leads to relative succinctness results between logics. For lack of space, we shall not discuss the general explainability problem further in the current article, but
instead leave the topic for the future.

\newcommand{\system}[1]{\emph{#1}}
\newcommand{\code}[1]{\texttt{#1}}
\newcommand{\union}{\cup}
\newcommand{\simplify}[2]{{#1}\!\mid_{#2}}

\section{Implementation}
\label{section:implementation}
\begin{lstlisting}[label=code:check-pos,frame=single,float=t,basicstyle=\ttfamily\footnotesize,escapechar=?,caption={Checking Positive Precondition (Lemma \ref{fullsize})}]
:- clause(C), nlit(P): pcond(C,P); plit(N): ncond(C,N). ?\label{line:counter-ex}?
simp(C) :- plit(P), pcond(C,P).      simp(C) :- nlit(N), ncond(C,N). ?\label{line:sat-by-pos-or-neg}?
simp(C) :- pcond(C,A), ncond(C,A). ?\label{line:tautology}? 
:- clause(C), not simp(C), pcond(C,P), not plit(P), not nlit(P). ?\label{line:open-pos}?
:- clause(C), not simp(C), ncond(C,N), not plit(N), not nlit(N). ?\label{line:open-neg}?
\end{lstlisting}
In this section, we devise a proof-of-concept implementation of explainability problems defined above. The implementation exploits the ASP fragment of the \system{Clingo} system%
\footnote{\url{https://potassco.org/clingo/}}
combining the \system{Gringo} grounder with the \system{Clasp} solver. However, we adopt the modular approach of \cite{Janhunen22} for the representation of oracles, thus hiding disjunctive rules and their saturation from encodings.
Since $\mathit{CNF}$-formulas are dominant in the context of SAT checking, we devise our first implementation under an assumption that input formulas take this form. Thus, in spirit of Lemma~\ref{fullsize}, $\varphi$ is essentially a set $L$ of literals and $\psi$ is a set $S$ of \emph{clauses}, i.e., disjunctions of literals. To enable meta-level encodings in ASP, a $\mathit{CNF}$-formula in DIMACS format can be reified into a set of first-order (ground) facts using the \system{lpreify} tool%
\footnote{\url{https://github.com/asptools/software}}
(option flag \code{-d}). Using additional rules, we define domain predicates \code{clause/1}, \code{pcond/2}, and \code{ncond/2} for identifying clauses, their positive conditions, and negative conditions respectively. The literals in the set $L$ are expressed by using domain predicates \code{plit/1} and \code{nlit/1} for positive and negative literals, respectively.

We relax the requirement that the set of literals $L$ is maximal, so that any three-valued interpretation can be represented. However, the precondition for the positive (resp.~negative) explanation is essentially the same: the result $\simplify{S}{L}$ of \emph{partially evaluating} $S$ with respect to $L$ must remain valid (resp.~unsatisfiable) in accordance to Lemma \ref{fullsize}.

The positive check is formalized in Listing~\ref{code:check-pos}. The constraint in Line \ref{line:counter-ex} excludes the possibility that $L$ falsifies $S$ directly. Lines \ref{line:sat-by-pos-or-neg}--\ref{line:tautology} detect which clauses of $S$ are immediately true given $L$ and removed from $\simplify{S}{L}$ altogether. Rules in Lines \ref{line:open-pos} and \ref{line:open-neg} deny any clause containing yet open literals that could be used to falsify the clause in question. The net effect is that the encoding extended by facts describing $L$ and $S$ has an answer set iff $\simplify{S}{L}$ is valid. Since the scope of negation is restricted to domain predicates only, the check is effectively polytime.
\begin{lstlisting}[label=code:check-neg,frame=single,float=t,basicstyle=\ttfamily\footnotesize,escapechar=?,caption={Checking Negative Precondition (Lemma \ref{fullsize})}]
t(A) :- plit(A), atom(A). ?\label{line:is-true}?
{ t(A) } :- atom(A), not plit(A), not nlit(A). ?\label{line:is-undef}? 
:- clause(C), not t(P): pcond(C,P); t(N): ncond(C,N). ?\label{line:clause}? 
\end{lstlisting}
The negative case can be handled by a single ASP program evaluating a coNP query, see Listing \ref{code:check-neg}. The rule in Line \ref{line:is-true} infers any positive literal in $L$ to be true while the negative ones in $L$ remain false \emph{by default}. In Line \ref{line:is-undef}, the truth values of atoms undefined in $L$ are freely chosen. The constraint in Line \ref{line:clause} ensures that each clause in the input $S$ must be satisfied. Thus $\simplify{S}{L}$ is unsatisfiable iff the encoding extended by facts describing $L$ and $S$ has no answer set. In general, this check is deemed worst-case exponential, but for maximal $L$, the task reduces to simple polytime propagation as no choices are active in Line \ref{line:is-undef}.

Our more general goal is to find \emph{minimum-size} explanations $L'\subseteq L$ possessing the identical property as required from $L$, i.e., $\simplify{S}{L'}$ is either valid or unsatisfiable. As regards the size of $L'$, we leave the mapping back to a minimum-size formula as a post-processing step.
In the negative case (the second item of Lemma~\ref{fullsize}), the idea is formalized by Listing~\ref{code:minimize}. While the (fixed) set $L$ is specified as before using predicates \code{plit/1} and \code{nlit/1}, the subset $L'$ is determined by choosing $L$-compatible truth values for atoms in Line \ref{line:pick-true-or-false}. The resulting size of $L'$ is put subject to minimization in Line \ref{line:optimize} if \code{k=0}, as set by default in Line \ref{line:parameter}. Positive values \code{k>0} set by the user activate the \emph{special explainability} mode: the size of $L'$ is at most \code{k} by the cardinality constraint in Line \ref{line:is-explainable}.
\begin{lstlisting}[label=code:minimize,frame=single,float=t,basicstyle=\ttfamily\footnotesize,escapechar=?,caption={Finding Minimum-Size or Bounded-Size Explanations}]
{ t(A) } :- atom(A), plit(A).     { f(A) } :- atom(A), nlit(A). ?\label{line:pick-true-or-false}?
#const k=0. ?\label{line:parameter}?
#minimize { 1,A: t(A), k=0; 1,A: f(A), k=0}. ?\label{line:optimize}?
:- #count { A: t(A); A: f(A)} > k, k>0. ?\label{line:is-explainable}?
\end{lstlisting}
%
Besides this objective, we check that $L'\union S$ is unsatisfiable by using an oracle encoded in Listing \ref{code:neg-oracle}. The \emph{input atoms} (cf.~\cite{Janhunen22}) are declared in Line \ref{line:inputs}. The predicate \code{et/1} captures a two-valued truth assignment compatible with $L'$ as enforced by Lines \ref{line:extension} and \ref{line:completeness}. Moreover, the clauses of $S$ are satisfied by constraints introduced in Line \ref{line:satisfy}. Thus, the oracle has an answer set iff $L'\union S$ is satisfiable. However, \emph{stable-unstable} semantics \cite{BogaertsJT16} and the translation \emph{unsat2lp} from \cite{Janhunen22} yield the complementary effect, amounting to the unsatisfiability of $\simplify{S}{L'}$.
\begin{lstlisting}[label=code:neg-oracle,frame=single,float=t,basicstyle=\ttfamily\footnotesize,escapechar=?,caption={Oracle for the Negative Case}]
{ t(A) } :- plit(A).          { f(A) } :- nlit(A). ?\label{line:inputs}?
et(A) :- t(A). ?\label{line:extension}?
{ et(A) }:- not t(A), not f(A), atom(A). ?\label{line:completeness}?
:- clause(C), not et(P): pcond(C,P); et(N): ncond(C,N). ?\label{line:satisfy}?
\end{lstlisting}

On the other hand, the positive case (the first item of Lemma \ref{fullsize}), can be covered by extending the program of Listing \ref{code:minimize} by further rules in Listing \ref{code:pos-oracle}.
\begin{lstlisting}[label=code:pos-oracle,frame=single,float=t,basicstyle=\ttfamily\footnotesize,escapechar=?,caption={Extension for the Positive Case}]
simp(C) :- t(P), pcond(C,P).          simp(C) :- f(N), ncond(C,N). ?\label{line:sat-by-pos-or-neg2}?
simp(C) :- pcond(C,A), ncond(C,A). ?\label{line:tautology2}?
:- clause(C), not simp(C), pcond(C,P), not t(P), not f(P). ?\label{line:open-pos2}?
:- clause(C), not simp(C), ncond(C,N), not t(N), not f(N). ?\label{line:open-neg2}?
\end{lstlisting}
%
The rules are analogous to those in Listing \ref{code:check-pos}, but formulated in terms of predicates \code{t/1} and \code{f/1} rather than \code{plit/1} and \code{nlit/1}. Thus $L'$ inherits the properties of $L$, i.e., the encoding based on Listings \ref{code:minimize} and \ref{code:pos-oracle} extended by facts describing $L$ and $S$ has an answer set iff $\simplify{S}{L'}$ is valid for a minimum-size $L'\subseteq L$.


\section{Experiments}
\label{section:experiments}

In what follows, we evaluate the computational performance of the \system{Clingo} system by using the encodings from Section \ref{section:implementation} and two benchmark problems,%
\footnote{\url{https://github.com/asptools/benchmarks}}
viz. the famous $n$-queens ($n$-Qs) problem and the dominating set (DS) problem of undirected graphs---recall the proof of Theorem~\ref{thm:cnf-formulas} in this respect. Besides understanding the scalability of Clingo in reasoning tasks corresponding to explanation problems defined in this work, we get also some indications what kinds of explanations are obtained in practice.
We study explainability in the context of these benchmark problems to be first encoded as SAT problems in $\mathit{CNF}$ using the declarative approach from \cite{GebserJKST16}: clauses involved in problem specifications are stated with rules in ASP style, but interpreted in $\mathit{CNF}$ by using an adapter called \emph{Satgrnd}. Thus, the \system{Gringo} grounder of \system{Clingo} can be readily used for the instantiation of the respective propositional schemata, as given in Listings~\ref{code:queens} and~\ref{code:dominating}, for subsequent SAT solving.
 
\begin{lstlisting}[label=code:queens,frame=single,float=t,basicstyle=\ttfamily\footnotesize,escapechar=?,caption={Propositional Specification for the $n$-Queens problem}]
#const n=8. ?\label{line:number-of-queens}?
pair(X1,X2) :- X1=1..n, X2=X1+1..n. ?\label{line:pairs}?
triple(X1,X2,Y1) :- pair(X1,X2), Y1=1..n-(X2-X1). ?\label{line:triples}?
queen(X,Y): Y=1..n :- X=1..n. ?\label{line:position}?
-queen(X,Y1) | -queen(X,Y2) :- X=1..n, pair(Y1,Y2). ?\label{line:mutex-row}?
-queen(X1,Y) | -queen(X2,Y) :- pair(X1,X2), Y=1..n. ?\label{line:mutex-col}?
-queen(X1,Y1) | -queen(X2,Y1+X2-X1) :- triple(X1,X2,Y1). % X1+Y2=X2+Y1 ?\label{line:mutex-diagonal1}?
-queen(X1,Y1+X2-X1) | -queen(X2,Y1) :- triple(X1,X2,Y1). % By symmetry ?\label{line:mutex-diagonal2}?
\end{lstlisting}
%
Our $n$-Qs encoding in Listing \ref{code:queens} introduces a default value for the number of queens \code{n} in Line~\ref{line:number-of-queens} and, based on \code{n}, the pairs and triples of numbers relevant for the construction of clauses are formed in Lines \ref{line:pairs} and \ref{line:triples}.
Then, for the queen in a column \code{X}, the length \code{n} clause in Line~\ref{line:position} chooses a row \code{Y} made unique for \code{X} by the clauses introduced in Line \ref{line:mutex-row}. Similarly, columns become unique by the clauses in Line \ref{line:mutex-col}. Finally, queens on the same diagonal are denied by clauses resulting from Lines \ref{line:mutex-diagonal1} and \ref{line:mutex-diagonal2}.
\begin{lstlisting}[label=code:dominating,frame=single,float=t,basicstyle=\ttfamily\footnotesize,escapechar=?,caption={Propositional Specification for the Dominating Set Problem}]
vertex(X) :- edge(X,Y).     vertex(Y) :- edge(X,Y). ?\label{line:vertices}?
in(X) | in(Y): edge(X,Y) | in(Z): edge(Z,X) :- vertex(X). ?\label{line:in}?
\end{lstlisting}
Turning our attention to the DS problem in Listing \ref{code:dominating}, vertices are extracted from edges in Line \ref{line:vertices}. The clauses generated in Line \ref{line:in} essentially capture the disjunctions collected as parts of $\psi$ in \eqref{eq:dominating}: our encoding assumes that the edges of the graph are provided as ordered pairs for the sake of space efficiency. Intuitively, given any vertex \code{X} in the graph, either it is \emph{in} the (dominating) set or any of its neighboring vertices is.
 
In the experiments, we evaluate the performance of the \system{Clasp} solver (v. 3.3.5) when used to solve various explainability problems. All test runs are executed on a cluster of Linux machines with Intel Xeon 2.40 GHz CPUs, and a memory limit of 16 GB. We report only the running times of the solver, since the implementations of grounding and translation steps are suboptimal due to our meta-level approach and such computations could also be performed off-line in general, and it is worth emphasizing that our current method has been designed to work for any $\mathit{CNF}$ and set of literals given as input.
For $n$-Qs, we generate (i) \emph{positive} instances by searching for random solutions to the problem with different values of $n$ and (ii) \emph{negative} instances by moving, in each solution found, one randomly selected queen to a wrong row. The respective truth assignments are converted into sets of literals $L$ ready for explaining.
For DS, we first generate random planar graphs of varying sizes and search for minimum-size dominating sets for them. Then, \emph{negative} instances are obtained by moving one random vertex outside each optimal set and by describing the outcomes as sets of literals $L$. For \emph{positive} instances, we include the positive literal \code{in(X)} in $L$ for all vertices \code{X}, in analogy to the reduction deployed in Theorem \ref{thm:cnf-formulas}.

The results of our experiments are collected in Figure \ref{fig:results}. Since initial screening suggests that explanation under exact size bounds is computationally more difficult, we present only results obtained by minimizing the size of $L$ using the \system{Clasp} solver in its \emph{unsatisfiable core} (USC) mode.
Figure~\ref{fig:results-a} shows the performance of \system{Clasp} when searching for positive explanations for DS based on planar graphs from $100$ up to $550$ vertices. Explanations are minimum-size dominating sets.
The performance obtained for the respective negative explanations is presented in Figure~\ref{fig:results-b}. In spite of somewhat similar scaling, far smaller planar graphs with the number of vertices in the range $10\ldots 50$ can be covered. In this case, explanations consist of sets of vertices based on some vertex and its neighbors.
The final plot in Figure~\ref{fig:results-c} concerns $n$-Qs when $n=8\ldots 24$ and negative explanations are sought. Explanations obtained from the runs correspond to either (i) single (misplaced) queens or (ii) pairs of (threatening) queens. The respective positive instances simply reproduce solutions and are computationally easy. Therefore, they are uninteresting.
 \begin{figure}[t]
 \vspace{-20pt}
 \hfill
\subfloat[\label{fig:results-a}]{\includegraphics[width=0.33\linewidth]{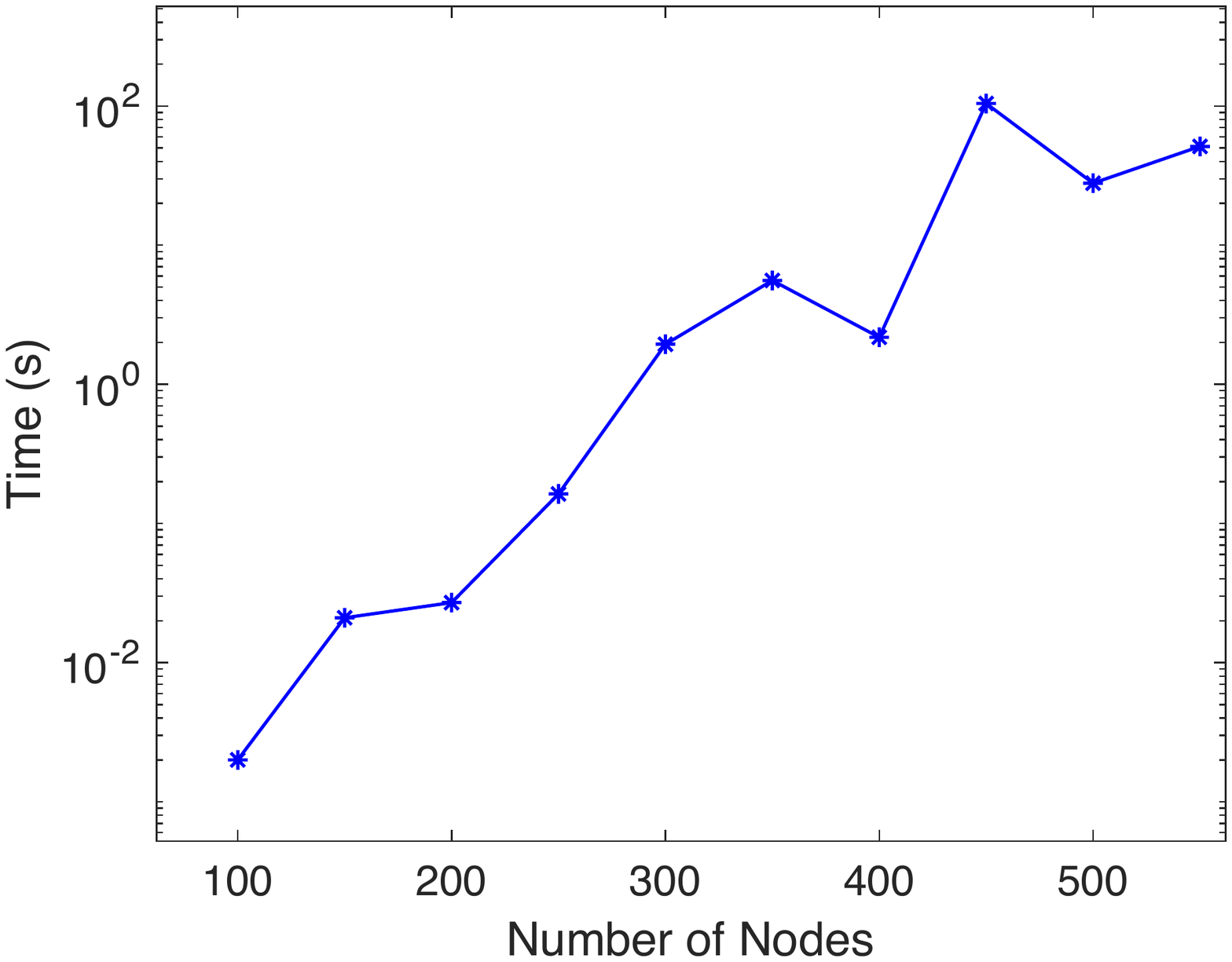}}
\hfill
\subfloat[\label{fig:results-b}]{\includegraphics[width=0.33\linewidth]{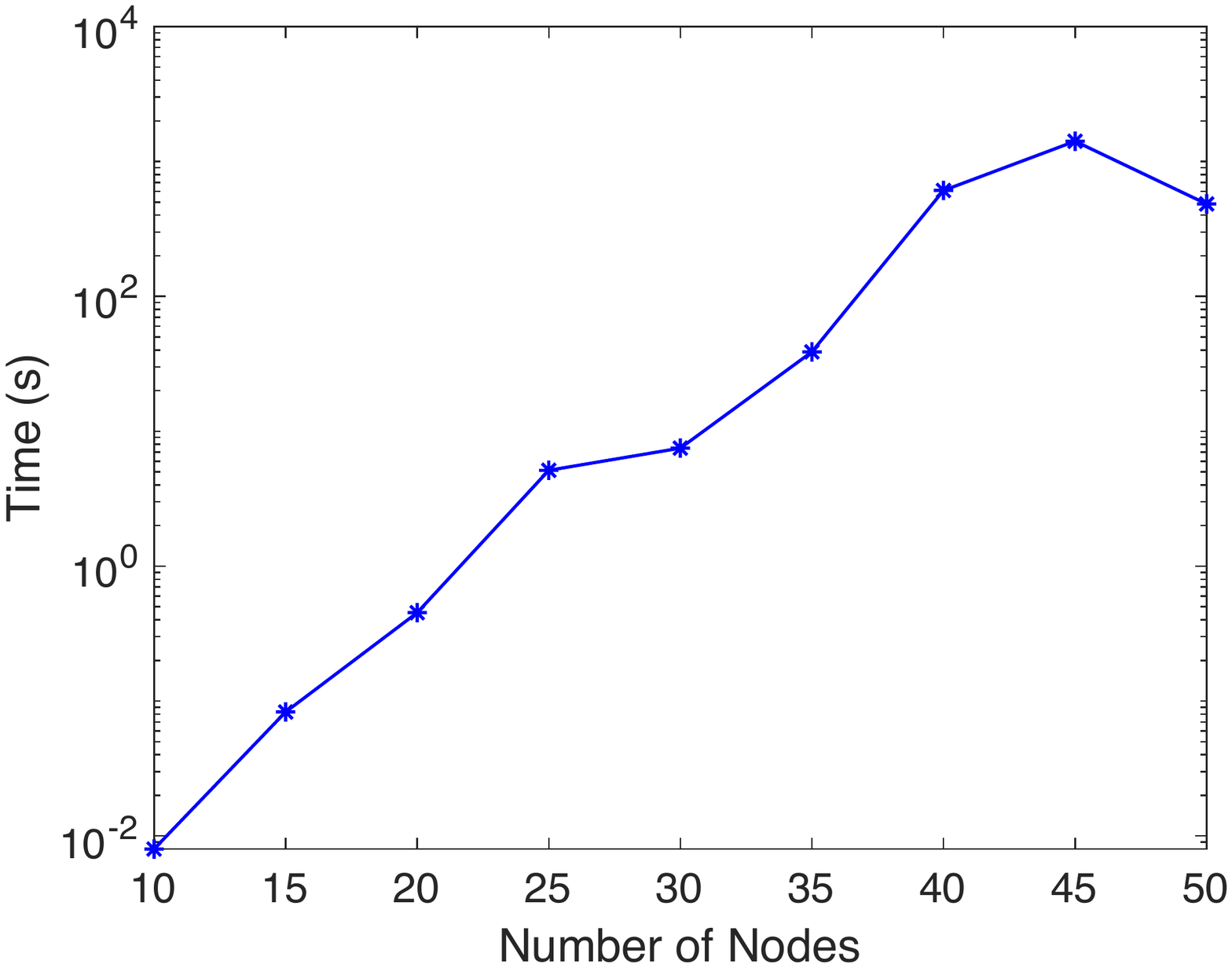}}
\hfill
\subfloat[\label{fig:results-c}]{\includegraphics[width=0.33\linewidth]{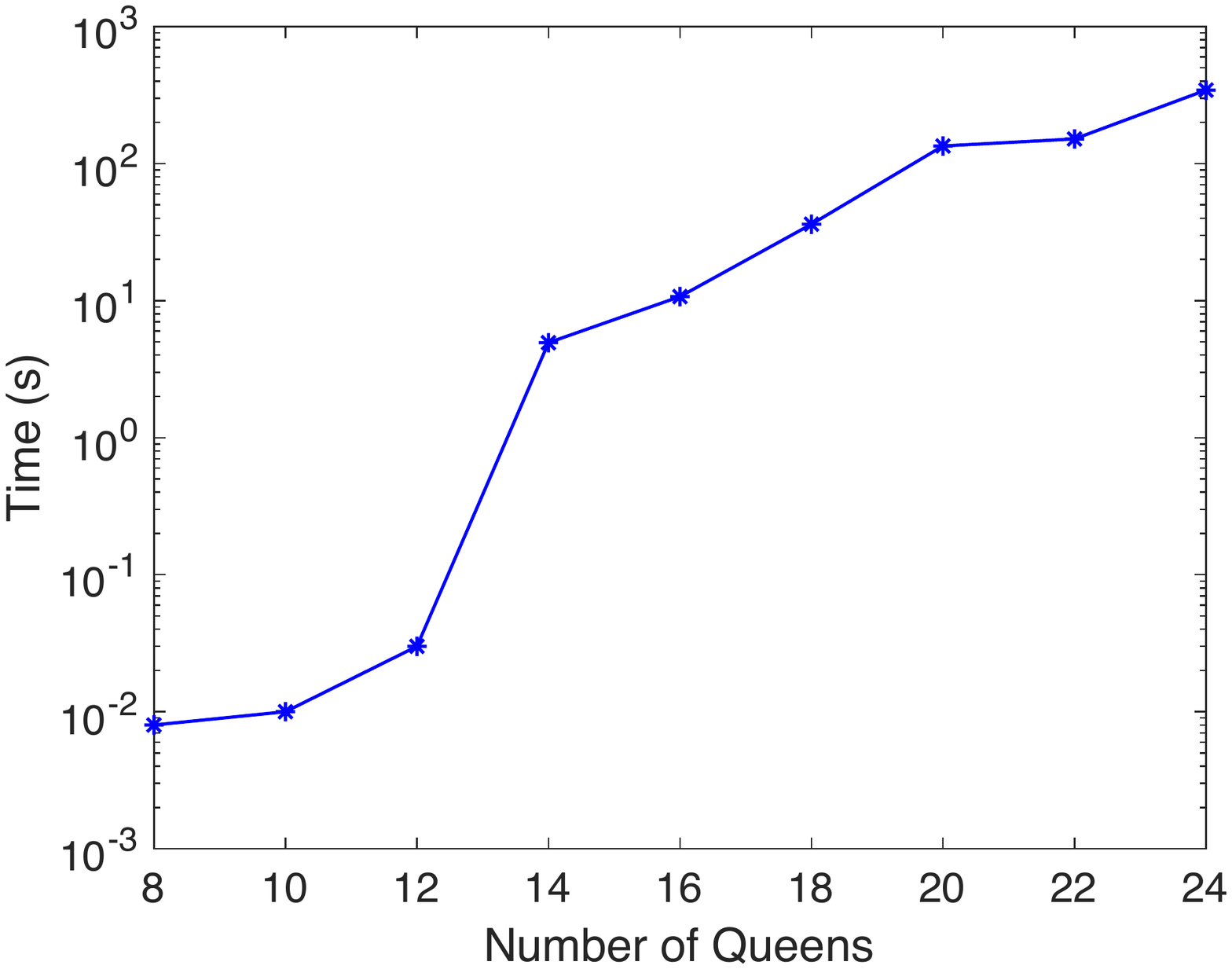}}
\hfill
\caption{Experimental Results on Explaining DS positively (a), negatively (b), and $n$-Qs negatively (c).\label{fig:results}}
\end{figure}

Some observations are in order. The instances obtained by increasing their size, i.e., either the number of vertices or queens, give rise to higher running times almost systematically. For each size, the time is computed as an average for running $10$ instances of equal size. Due to logarithmic scale, running times tend to scale exponentially. Moreover, positive explanations appear to be easier to find than negative ones in compliance with complexity results.


\section{Conclusion}
\label{section:conclusion}

%
We have provided general, logic-based definitions of explainability and studied the particular case of propositional logic in detail. The related $\Sigma_2^P$-complenetess result gives a useful, robust lower bound for a wide range of more expressive logics and future work. We have also shown \textsc{NP}-completeness of the explainability problems with formulas in $\mathit{CNF}$ and $\mathit{DNF}$ when the input truth value is restricted. Moreover, we have presented a proof-of-concept implementation for the explanation of  $\mathit{CNF}$-formulas (without truth value restrictions).
Our experimental results confirm the expected worst-case exponential runtime behaviour of \system{Clasp}. Negative explanations have higher computational cost than positive explanations. The optimization variants of explanation problems seem interesting, because the USC strategy seems very effective and the users need not provide fixed bounds for queries in advance.

We note that short formulas or expressions can often be natural in explaining more complex ones. However, it is clear that in various settings, short expressions can be difficult to parse. Up to an extent, this phenomenon is taken into account in Definition \ref{logictwo}, where we let the complexity function $m$ depend on $\mathcal{M}$ as well as $\mathcal{F}$. However, further important directions remain to be investigated. A central issue here is considering normal forms of logics, as such forms are often custom-made such that shorter formulas are \emph{not} difficult to parse, but instead reveal the meaning of the formula directly.

\medskip

\medskip 

\medskip 

\medskip

\medskip 

\medskip

\noindent
\textbf{Acknowledgments.}\ \ \ 
Tomi Janhunen, Antti Kuusisto, Masood Feyzbakhsh Rankooh and Miikka Vilander were supported by the Academy of Finland consortium project \emph{Explaining AI via Logic} (XAILOG), grant numbers 345633 (Janhunen) and 345612 (Kuusisto). Antti Kuusisto was also supported by the Academy of Finland project \emph{Theory of computational logics}, grant numbers 324435, 328987 (to December 2021) and 352419, 352420 (from January 2022). The author names of this article have been ordered on the basis of alphabetic order.

\bibliographystyle{plain}
\bibliography{references}

\begin{thebibliography}{1}

\bibitem{BlackburnRV01}
Patrick Blackburn, Maarten de~Rijke, and Yde Venema.
\newblock {\em Modal Logic}, volume~53 of {\em Cambridge Tracts in Theoretical
  Computer Science}.
\newblock Cambridge University Press, 2001.

\bibitem{BogaertsJT16}
Bart Bogaerts, Tomi Janhunen, and Shahab Tasharrofi.
\newblock Stable-unstable semantics: Beyond {NP} with normal logic programs.
\newblock {\em Theory Pract. Log. Program.}, 16(5-6):570--586, 2016.

\bibitem{BuchfuhrerU11}
David Buchfuhrer and Christopher Umans.
\newblock The complexity of {Boolean} formula minimization.
\newblock {\em J. Comput. Syst. Sci.}, 77(1):142--153, 2011.

\bibitem{EiterG95}
Thomas Eiter and Georg Gottlob.
\newblock On the computational cost of disjunctive logic programming:
  Propositional case.
\newblock {\em Ann. Math. Artif. Intell.}, 15(3-4):289--323, 1995.

\bibitem{GebserJKST16}
Martin Gebser, Tomi Janhunen, Roland Kaminski, Torsten Schaub, and Shahab
  Tasharrofi.
\newblock Writing declarative specifications for clauses.
\newblock In Loizos Michael and Antonis~C. Kakas, editors, {\em {JELIA}}, pages
  256--271, 2016.

\bibitem{Janhunen22}
Tomi Janhunen.
\newblock Implementing stable-unstable semantics with {ASPTOOLS} and clingo.
\newblock In James Cheney and Simona Perri, editors, {\em {PADL}}, pages
  135--153, 2022.

\bibitem{ShihCD18}
Andy Shih, Arthur Choi, and Adnan Darwiche.
\newblock A symbolic approach to explaining {Bayesian} network classifiers.
\newblock In J{\'{e}}r{\^{o}}me Lang, editor, {\em {IJCAI}}, pages 5103--5111,
  2018.

\bibitem{SimonsNS02}
Patrik Simons, Ilkka Niemel{\"{a}}, and Timo Soininen.
\newblock Extending and implementing the stable model semantics.
\newblock {\em Artif. Intell.}, 138(1-2):181--234, 2002.

\bibitem{Umans01}
Christopher Umans.
\newblock The minimum equivalent {DNF} problem and shortest implicants.
\newblock {\em J. Comput. Syst. Sci.}, 63(4):597--611, 2001.

\end{thebibliography}


\end{document}